\documentclass[11pt]{article}
\usepackage{amsmath,amssymb,amsthm,amsfonts,latexsym,xspace,enumerate,graphicx,color,eucal,upgreek,mathtools}
\usepackage[backref,colorlinks,bookmarks=true]{hyperref}
\usepackage{float}

\newtheorem{theorem}{Theorem}[section]
\newtheorem*{namedtheorem}{\theoremname}
\newcommand{\theoremname}{testing}

\newtheorem{lemma}[theorem]{Lemma}

\newtheorem{claim}[theorem]{Claim}

\newtheorem{corollary}[theorem]{Corollary}

\newtheorem*{question*}{Question}

\theoremstyle{definition}
\newtheorem{definition}[theorem]{Definition}

\theoremstyle{plain}
\newtheorem{Alg}{Algorithm}

\renewenvironment{proof}{\noindent{\textbf{Proof:}}} {$\blacksquare$\vskip \belowdisplayskip}

\newcommand{\ignore}[1]{}

\newcommand{\E}{\mathop{\bf E\/}}

\newcommand{\Var}{\mathop{\bf Var\/}}

\newcommand{\poly}{\mathrm{poly}}

\newcommand{\R}{\mathbb R}

\newcommand{\calT}{\mathcal{T}}

\newcommand{\norm}[1]{\left\lVert #1 \right\rVert}

\makeatletter
\floatstyle{ruled}
\newfloat{fragment}{H}{lop}
\floatname{fragment}{Algorithm}
\renewcommand{\floatc@ruled}[2]{\vspace{2pt}{\@fs@cfont #1.\:} #2 \par
 \vspace{1pt}}
\makeatother

\newcommand{\mypseudocodelabel}[1]{\hfil}

\title{{\huge Learning Topic Models --- Going beyond SVD}}
\author{Sanjeev Arora \thanks{Princeton University, Computer Science Department and Center for Computational Intractability.
Email: {\tt arora@cs.princeton.edu}. This work is supported by the NSF grants CCF-0832797 and CCF-1117309.} \and
Rong Ge \thanks{Princeton University, Computer Science Department and Center for Computational Intractability.
Email: {\tt rongge@cs.princeton.edu}. This work is supported by the NSF grants CCF-0832797 and CCF-1117309.} \and
Ankur Moitra \thanks{Institute for Advanced Study, School of Mathematics.
Email: {\tt moitra@ias.edu}.
Research supported in part 
by NSF grant No.
DMS-0835373 and by an NSF Computing and Innovation Fellowship.}
}
\begin{document}
\maketitle

\begin{abstract}
{\em Topic Modeling} is an approach used for automatic comprehension
and classification of data in a variety of settings, and perhaps the canonical application is in
uncovering thematic structure in a corpus of documents. A number of foundational works both in machine learning \cite{Hof} and in theory \cite{PRTV} have suggested a
probabilistic model for documents, whereby documents arise as a convex combination of (i.e. distribution on) a small number of {\em
  topic} vectors, each topic vector being  a distribution on words
(i.e. a vector of word-frequencies). Similar models have since been used in
a variety of application areas; the {\em
  Latent Dirichlet Allocation} or LDA model of Blei et al.\ is especially popular.

Theoretical studies of topic modeling focus on learning the model's parameters
{\em assuming the data is actually generated from it.} Existing
approaches for the most part rely on {\em Singular Value Decomposition} (SVD), and
consequently have one of two limitations: these works need to either assume that each document contains only one topic, or
else can only recover the {\em span} of the topic vectors instead of the
topic vectors themselves. 

This paper formally justifies {\em Nonnegative Matrix Factorization} (NMF) as
a main tool in this context,  which is  an analog of SVD
where all vectors are nonnegative. Using this tool we give the first 
polynomial-time algorithm for learning topic models without the above
two limitations. The algorithm uses a fairly mild assumption about the
underlying topic matrix called {\em separability}, which is
usually found to hold in real-life data.  A compelling feature of our algorithm is that it
generalizes to models that incorporate topic-topic correlations, such as the
{\em Correlated
Topic Model} (CTM) and the {\em Pachinko Allocation Model} (PAM).

We hope that this paper will motivate further theoretical results that use NMF as a replacement for SVD -- just
as NMF has come to replace SVD in many applications. 
\end{abstract}

\setcounter{page}{0} \thispagestyle{empty}

\newpage

\section{Introduction}\label{sec:back}


Developing tools for automatic comprehension and classification of data ---web
pages, newspaper articles, images, genetic sequences, user ratings --- is a holy
grail of machine learning. {\em Topic Modeling} is an approach that
has proved successful in all of the aforementioned settings, though for concreteness here we will focus
 on uncovering thematic structure of a corpus of documents
(see e.g. \cite{survey}, \cite{LDA}). 

In order to learn structure one has to posit the {\em existence} of
structure, and in topic models one assumes a {\em generative model}
for a collection of documents. Specifically, each document  is represented
as a vector of word-frequencies (the {\em bag of words} representation).
Seminal papers in theoretical CS (Papadimitriou et al.~\cite{PRTV}) and
machine learning (Hofmann's {\em Probabilistic Latent Semantic
  Analysis}~\cite{Hof}) suggested that
documents arise as a convex combination of (i.e. distribution on) a small number of {\em
  topic} vectors, where each topic vector is a distribution on words
(i.e. a vector of word-frequencies). Each convex combination of topics
thus is itself a distribution on words, and the document is assumed to
be generated by drawing $N$ independent samples from it.  Subsequent
work makes specific choices for the distribution used to generate topic
combinations ---the well-known {\em Latent Dirichlet Allocation} (LDA) model of Blei et al~\cite{LDA} 
hypothesizes a {\em Dirichlet} distribution (see Section~\ref{sec:dirichlet}). 

In machine learning, the prevailing approach is to use local search (e.g. \cite{DLR}) or
other heuristics \cite{WJ} in an attempt to find a {\em maximum likelihood} fit to
the above model.  For example, fitting to a corpus of newspaper articles may
reveal $50$ topic vectors corresponding to, say, politics, sports,
weather, entertainment etc., and a particular article 
could be explained as a $(1/2, 1/3, 1/6)$-combination of the topics
politics, sports, and entertainment.  Unfortunately (and not surprisingly), the maximum
likelihood estimation is $NP$-hard (see Section~\ref{sec:mle}) and
consequently when using this paradigm, it seems necessary to rely on unproven
heuristics even though these have well-known limitations (e.g. getting stuck
in a local minima \cite{DLR,RW}). 



The  work of  Papadimitriou et al~\cite{PRTV} (which also
formalized the topic modeling problem) and a long line of subsequent work
have attempted to give {\em provable}
guarantees for the problem of learning the model parameters
{\em assuming the data is actually generated from it.} This is in
contrast to a maximum likelihood approach, which asks to find the
closest-fit model for arbitrary data. The principal algorithmic
problem is the following (see Section~\ref{sec:our} for more details):

\begin{quotation}
\noindent {\bf Meta Problem in Topic Modeling:} {\em There is an unknown topic matrix
  $A$ with nonnegative entries that is dimension $n\times r$,
  and a stochastically generated unknown matrix $W$ that is dimension $r \times
  m$. Each column of $AW$ is viewed as a probability distribution on rows,
  and for each column we are given $N \ll n$ i.i.d. samples from the associated distribution.}
  
  \bigskip

\noindent {\bf Goal:} {\em Reconstruct $A$ and parameters
  of the generating distribution for $W$.}
\end{quotation}

The challenging aspect of this problem is that we wish to recover {\em nonnegative} matrices
$A, W$ with small inner-dimension $r$.  
The general problem of finding nonnegative factors $A, W$
of specified dimensions when given the matrix $AW$ (or a close approximation)
 is called the {\em Nonnegative Matrix Factorization} (NMF) problem~(see~\cite{LS99}, and \cite{AGKM} for a longer history) and it is NP-hard~\cite{Vav}.
Lacking a tool to solve such  problems, theoretical
work has generally relied on the {\em Singular Value Decomposition}
(SVD) which given the matrix $AW$ will instead find factors $U, V$ with both positive and negative
entries.  
SVD can be used as a tool for clustering -- in which case one needs to assume that each document has only {\em one} 
topic. In Papadimitriou et
al \cite{PRTV} this is called the {\em pure documents} case and is solved
under strong assumptions about the topic matrix $A$ (see also~\cite{McS} and \cite{AHK} which uses the method of moments instead). 
Alternatively, other papers use
SVD to recover the {\em span} of the columns of $A$ (i.e. the
topic vectors)~\cite{AFKMS}, \cite{KS1}, \cite{KS2}, which suffices
for some applications such as computing the inner product of two document vectors (in the space spanned
by the topics) as a measure of their {\em similarity}. 



These limitations of existing approaches ---either restricting to one
topic per document, or else learning only the span of the topics instead of
the topics themselves---are quite serious. In practice documents are much more faithfully described as a distribution on topics and indeed for a wide range of applications one needs the actual topics and not just their span -- such as when browsing a collection of documents without a particular query
phrase in mind, or tracking how topics evolve over time (see  \cite{survey} for a survey of various applications). Here we consider what we believe to be a much
weaker assumption -- {\em separability}. Indeed, this property has already been identified as a natural one in the machine
learning community \cite{DS} and has been empirically observed to hold in topic
matrices fitted to  various types of data~\cite{Blei}. 

Separability requires that each topic has some near-perfect indicator word -- a word that we call
the {\em anchor word} for this topic--- that appears with reasonable
probability in that topic but with negligible probability in all other topics
(e.g., ``soccer'' could be an anchor word for the topic
``sports''). We give a formal definition in Section~\ref{sec:our}. This property is particularly natural in the context of topic modeling, 
where the number of distinct words (dictionary size) is very large compared to the number
of topics. In a typical application, it is common to have a dictionary size in the thousands or tens of thousands,
but the number of topics is usually somewhere in the range from $50$ to $100$. 
Note that separability does {\em not} mean that the anchor word always occurs
(in fact, a typical document  may be very likely to contain {\em no}
anchor words). Instead, it dictates that when an anchor word does occur, it
is a strong indicator that the corresponding topic is in the
mixture used to generate the document.

Recently, we gave a polynomial time algorithm to solve NMF under the condition that the topic
matrix $A$ is separable \cite{AGKM}. The intuition that underlies this algorithm is that the set of anchor words
can be thought of as extreme points (in a geometric sense) of the
dictionary. This condition can be used to identify all of the anchor
words and then also the nonnegative factors.
Ideas from this algorithm are a key ingredient in our present paper, but our focus is on the question:

\begin{question*}
What if we are not given the true matrix $AW$, but are instead given a few samples (say, $100$ samples) from the distribution represented
by each column? 
\end{question*}

The main technical challenge in adapting our earlier NMF algorithm
 is that each document vector is a {\em very poor} approximation to the corresponding column
of $AW$ ---it is {\em too noisy} in any reasonable measure of
noise. Nevertheless, the core insights of our
NMF algorithm still apply.  
Note that it is impossible to learn the matrix $W$
to within arbitrary accuracy. (Indeed, this is information
theoretically impossible even if we knew the topic matrix $A$ and the distribution
from which the columns of $W$ are generated.) So we {\em cannot} in
general give an estimator that converges to the true matrix $W$, and
yet we {\em can} give an estimator that converges to the
true topic matrix $A$! (For an overview of our algorithm, see the
first paragraph of Section~\ref{sec:mainalg}.)

We hope that this application of our NMF algorithm is just a starting point
and other theoretical results will start using NMF as a replacement for SVD -- just
as NMF has come to replace SVD in several applied settings.

\subsection{Our Results}\label{sec:our}

Now we precisely define the topic modeling (learning) problem which
was informally introduced above.
There is an unknown {\em topic matrix} $A$ which is dimension $n \times r$ (i.e. $n$ is the
dictionary size) and each column of $A$ is a distribution on
$[n]$. There is an unknown $r \times m$ matrix $W$ whose each column
is itself a distribution (aka convex combination) on $[r]$. The
columns of $W$ are
i.i.d.\ samples from a distribution $\mathcal{T}$ which belongs to a known family, e.g.,
Dirichlet distributions, but whose parameters are unknown. Thus each
column of $AW$ (being a convex combination of distributions) is itself a distribution
on $[n]$, and the algorithm's input consists of $N$ i.i.d. samples for
each column of $AW$.  Here $N$ is the document size and is assumed to be a constant for simplicity. Our algorithm can be easily adapted to work when the documents have different sizes.

The algorithm's running time will necessarily depend upon various
model parameters, since distinguishing a very small parameter from $0$
imposes a lower bound on the number of samples needed.
The first such parameter is a  quantitative version  of {\em separability}, 
which was presented above as a natural assumption
in context of topic modeling. 




\begin{definition} [$p$-Separable Topic Matrix] \label{def:psep} An $n\times r$
  matrix $A$ is {\em $p$-separable} if for each $i$  there is some row
  $ \pi(i)$ of $A$ that has a single nonzero entry which is in the
  $i^{th}$ column and it is at least $p$.
\end{definition}



The next parameter measures the lowest probability with which a topic
occurs in the distribution that generates columns of $W$. 

\begin{definition}[Topic Imbalance]
The {\em topic imbalance} of the model is the ratio between the
largest and smallest expected entries in a column of $W$, in other
words, $a =\max_{i,j\in [r]} \frac{\E[X_i]}{\E[X_j]}$ where $X \in \R^r$ is
a random weighting of topics chosen from the distribution.
\end{definition}

Finally, we require that topics stay identifiable despite
sampling-induced noise. To formalize this,
we define a matrix that will be important throughout this paper:


\begin{definition}[Topic-Topic Covariance Matrix $R(\mathcal{T})$] \label{def:varcovar}
If $\mathcal{T}$ is the distribution that generates the columns of $W$, then 
$R(\mathcal{T})$ is defined as an $r \times r$ matrix whose $(i, j)$th
entry is $E[X_i X_j]$ where $X_1, X_2, ... X_r$ is a vector chosen from $\calT$.
\end{definition}

Let $\gamma > 0$ be a lower bound on the 
$\ell_1$-condition number of the matrix $R(\mathcal{T})$. This is
defined in   Section~\ref{sec:cond}, but for a $r \times r$ matrix it is within a factor of $\sqrt{r}$ of the smallest singular value. 
Our algorithm will work for any $\gamma$, but the number of documents we require will depend (polynomially) on $1/\gamma$: 

\begin{theorem}[Main]\label{thm:main}
There is a polynomial time algorithm that learns the parameters of a topic model
if the number of documents is at least 
$$m = \max\left\{O\left(\frac{\log n \cdot a^4r^6}{\epsilon^2p^6\gamma^2 N}\right), O\left(\frac{\log r \cdot a^2 r^4}{\gamma^2}\right)\right\},$$
where the three numbers $a, p, \gamma$ are as defined above. 
The algorithm learns the topic-term matrix $A$ up to 
additive error $\epsilon$. Moreover, when the number of documents is also larger than $O\left(\frac{\log r\cdot r^2}{\epsilon^2}\right)$ the algorithm can learn the topic-topic covariance matrix $R(\mathcal{T})$ up to additive error $\epsilon$.
\end{theorem}

As noted earlier, we are able to recover the topic matrix even though
we do not always recover the parameters  of the column
distribution $\mathcal{T}$. In some special cases we can also recover the parameters of $\mathcal{T}$,
e.g. when this distribution is Dirichlet,
as happens in the popular {\em Latent Dirichlet Allocation} (LDA) model \cite{LDA,survey}. 
In Section~\ref{sec:conddir} we compute a lower bound on the $\gamma$
parameter for the Dirichlet distribution,
which allows us to
apply our main learning algorithm, and also the parameters of $\calT$ can
be recovered from the co-variance matrix $R(\calT)$ (see Section~\ref{sec:recovdir}).



Recently the basic LDA model has been refined
to allow correlation among different topics, which is more
realistic.  See
for example the {\em Correlated Topic Model} (CTM) \cite{BL1} and
the {\em Pachinko Allocation Model} (PAM) \cite{LM}. 
A compelling aspect of our algorithm is that it extends to these models as well: we can learn the topic matrix, even
though we cannot always identify $\mathcal{T}$. (Indeed, the distribution $\calT$ in the Pachinko is not even identifiable: two different sets of parameters can generate exactly the same distribution)

\paragraph{Comparison with existing approaches.} 
(i) We rely crucially on separability. But note that
this assumption is weaker in some sense than the assumptions in {\em all}
prior works that provably learn the topic matrix. They
assume a single topic per document, which can be seen as a
strong separability assumption about $W$ instead of $A$ ---in {\em every}
column of $W$ only one entry is nonzero. By contrast, separability
only assumes a similar condition for a negligible fraction ---namely,
$r$ out of $n$--- of rows of $A$. Besides, it is found to actually
hold in topic matrices found using current heuristics.
(ii) Needless to say, existing theoretical approaches for recovering topic matrix $A$ couldn't handle topic
correlations at all since they only allow one topic per document. 
(iii) We remark that prior approaches that learn the span of $A$ instead of $A$ 
needed strong concentration bounds on eigenvalues of random matrices, and 
thus require substantial document sizes (on the order of the number of words in the dictionary!). {\em By contrast we can work with
documents of $O(1)$ size.}

\vspace*{-0.1in}
\section{Tools for (Noisy) Nonnegative Matrix Factorization}

\label{sec:cond}






\subsection{Various Condition Numbers}

Central to our arguments will be various notions of matrices being ``far'' from
being low-rank.  The most interesting one for our purposes was introduced by Kleinberg and
Sandler \cite{KS2} in the context of collaborative filtering; and can be thought of as an
$\ell_1$-analogue to the smallest singular value of a matrix. 

\begin{definition}[$\ell_1$ Condition Number] If matrix $B$ has
  nonnegative entries and all rows sum to $1$ then its $\ell_1$
  Condition Number $\Gamma(B)$ is defined as:

\[
\Gamma(B) = \min_{\|x\|_1 = 1} \|xB\|_1.
\]

\noindent If $B$ does not have row sums of one then $\Gamma(B)$ is equal to $\Gamma(DB)$ where $D$ is the diagonal matrix such that $DB$ has row sums of one.
\end{definition} 

For example, if the rows of $B$ have disjoint support then $\Gamma(B) =1$ and in general the quantity $\Gamma(B)$ can be thought of a measure of how close two distributions on \emph{disjoint} sets of rows can be. Note that, if $x$ is an $n$-dimensional real vector, $\|x\|_2 \leq \|x\|_1 \leq \sqrt{n} \|x\|_2$ and hence (if $\sigma_{min}(B)$ is the smallest singular value of $B$), we have:
$$\frac{1}{\sqrt{n}} \sigma_{min}(B) \leq \Gamma(B) \leq \sqrt{m} \sigma_{min}(B).$$

The above notions of condition number will be most relevant in the context of the topic-topic covariance matrix $R(\mathcal{T})$. We shall always use $\gamma$ to denote the $\ell_1$ condition number of $R(\mathcal{T})$. The definition of condition number will be preserved even when we estimate the topic-topic covariance matrix using random samples. 

\begin{lemma}\label{lem:convergeR}
When $m> 5\log r/\epsilon_0^2$, with high probability the matrix $R = \frac{1}{m}WW^T$ is entry-wise close to $R(\mathcal{T})$ with error $\epsilon_0$. Further, when $\epsilon_0 < 1/4\gamma a r^2 $ where $a$ is topic imbalance, the matrix $R$ has $\ell_1$ condition number at least  $\gamma/2$.
\end{lemma}

\begin{proof}
Since $\E[W_iW_i^T] = R(\mathcal{T})$, the first part is just by Chernoff bound and union bound. The further part follows because $R(\mathcal{T})$ is 
$\gamma$ robustly simplicial, and the error can change the $\ell_1$ norm of $vR$ for any unit $v$ by at most $ar \cdot r\epsilon_0$. The extra factor $ar$ comes from the normalization to make rows of $R$ sum up to 1.
\end{proof}

In our previous work on nonnegative matrix factorization~\cite{AGKM} we defined  a different measure of ``distance'' from singular which is essential to the polynomial time algorithm for NMF:

\begin{definition}[$\beta$-robustly simplicial] \label{def:robustsimp}
If  each column of a matrix $A$ has unit $\ell_1$ norm, then we say it
is {\em $\beta$-robustly simplicial} if no column in $A$ has $\ell_1$ distance smaller than $\beta$ to the convex hull of the remaining columns in $A$. 
\end{definition}

The following claim clarifies the interrelationships of these latter
condition numbers.

\begin{claim}
\label{claim:conditionrelation}
(i) If $A$ is $p$-separable then $A^T$ has $\ell_1$ condition number at least $p$. (ii)
If  $A^T$ has all row sums equal to $1$ then $A$ is $\beta$-robustly
simplicial for $\beta =\Gamma(A^T)/2$.  
\end{claim}

We shall see that the $\ell_1$ condition number for product of matrices is at least the product of $\ell_1$ condition number. The main application of this composition is to show that the matrix $R(\mathcal{T}) A^T$ (or the empirical version $RA^T$) is at least $\Omega(\gamma p)$-robustly simplicial. The following lemma will play a crucial role in analyzing our main algorithm:

\begin{lemma}[Composition Lemma] \label{lem:compose}
If $B$ and $C$ are matrices with $\ell_1$ condition number $\Gamma(B)\ge \gamma$ and $\Gamma(C)\ge \beta$ , then $\Gamma(BC)$ is at least $\beta \gamma$. Specificially, when $A$ is $p$-separable the matrix $R(\mathcal{T})A^T$ is at least $\gamma p /2$-robustly simplicial.
\end{lemma}
\begin{proof}
The proof is straight forward because for any vector $x$, we know $\norm{xBC}_1 \le \Gamma(C)\norm{xB}_1 \le \Gamma(C)\Gamma(B)\norm{x}_1$. 
For the matrix $R(\mathcal{T}) A^T$, by Claim~\ref{claim:conditionrelation} we know the matrix $A^T$ has $\ell_1$ condition number at least $p$. Hence $\Gamma(R(\mathcal{T}) A^T)$ is at least $\gamma p$ and again by Claim~\ref{claim:conditionrelation} the matrix is $\gamma p/2$-robustly simplicial.
\end{proof}

\vspace*{-0.1in}
\subsection{Noisy Nonnegative Matrix Factorization under Separability}

A key ingredient is an approximate NMF
algorithm from \cite{AGKM}, which
can recover an approximate nonnegative matrix factorization $\tilde{M} \approx AW$ when the $\ell_1$ distance between each row of $\tilde{M}$ and the corresponding row in $AW$ is small. We emphasize that this is not enough for our purposes, since the term-by-document matrix $\tilde{M}$ will have a substantial amount of noise (when compared to its expectation) precisely because the number of words in a document $N$ is much smaller than the dictionary size $n$. Rather, we will apply the following algorithm (and an improvement that we give in Section~\ref{sec:betternmf}) to the Gram matrix $\tilde{M}\tilde{M}^T$. 

\begin{theorem}[Robust NMF Algorithm~\cite{AGKM}]
\label{thm:separablenoise}

Suppose $M=AW$ where $W$ and $M$ are normalized to have rows sum up to 1,  
$A$ is separable and $W$ is $\gamma$-robustly simplicial. Let $\epsilon = O(\gamma^2)$. There is a polynomial time algorithm that given $\tilde{M}$ such that for all rows $\norm{\tilde{M}^i - M^i}_1<\epsilon$, finds a $W'$ such that $\norm{W'^i - W^i}_1 < 10\epsilon/\gamma+7\epsilon$. Further every row $W'^i$ in $W'$ is a row in $\tilde{M}$. The corresponding row in $M$ can be represented as $(1-O(\epsilon/\gamma^2))W^i + O(\epsilon/\gamma^2) W^{-i}$. Here $W^{-i}$ is a vector in the convex hull of other rows in $W$ with unit length in $\ell_1$ norm.

\end{theorem}

In this paper we need a slightly different goal here  than in \cite{AGKM}. Our goal is not to recover estimates to the anchor words that are close in $\ell_1$-norm but rather to recover almost anchor words (word whose row in $A$ has almost all its weight on a single coordinate). Hence, we will be able to achieve better bounds by treating this problem directly, and we give a substitute for the above theorem. We defer the proof to Section~\ref{sec:betternmf}. 

\begin{theorem}
\label{thm:betterseparablenoise}
Suppose $M=AW$ where $W$ and $M$ are normalized to have rows sum up to 1,  
$A$ is separable and $W$ is $\gamma$-robustly simplicial. When $\epsilon < \gamma/100$  there is a polynomial time algorithm that given $\tilde{M}$ such that for all rows $\|\tilde{M}^i - M^i\|_1<\epsilon$, finds $r$ row (almost anchor words) in $\tilde{M}$. The $i$-th almost anchor word corresponds to a row in $M$ that  can be represented as $(1-O(\epsilon/\gamma))W^i + O(\epsilon/\gamma) W^{-i}$. Here $W^{-i}$ is a vector in the convex hull of other rows in $W$ with unit length in $\ell_1$ norm.
\end{theorem}

\vspace*{-0.2in}
\section{Algorithm for Learning a Topic Model: Proof of Theorem~\ref{thm:main}}\label{sec:mainalg}

First it is important to understand why separability helps in nonnegative
matrix factorization, and specifically, the exact role played by the
anchor words. 
Suppose the NMF algorithm is given a matrix $AB$.
If $A$ is $p$-separable then this means that $A$ contains  a diagonal
matrix (up to row permutations). Thus a scaled copy of 
each row of $B$ is present as a row in $AB$. In fact, if we knew the anchor words of $A$,
then by looking at the corresponding rows of $AB$ we could``read off''
the corresponding row of $B$ (up to scaling), and use these in turn to
recover all of $A$. Thus the anchor words constitute the ``key'' that
``unlocks'' the factorization, and indeed the main step of our earlier
NMF algorithm was a geometric procedure to identify the anchor words.
When one is given a noisy version of $AB$, the analogous notion is
``almost anchor'' words, which correspond to rows of $AB$ that are ``very
close'' to rows of $B$; see Theorem~\ref{thm:betterseparablenoise}.

Now we sketch how to apply these insights to learning topic models.  Let $M$ denote the provided term-by-document
matrix, whose each column describes the empirical word frequencies in
the documents. It is obtained from sampling $AW$ and thus is an extremely noisy approximation to
$AW$. 
Our algorithm starts by forming the Gram matrix $MM^T$,
which can be thought of as an empirical word-word covariance matrix. In fact as
the number of documents increases $\frac{1}{m}MM^T$ tends to a limit
$Q = \frac{1}{m}E[AWW^T A],$ implying $Q = AR(\mathcal{T})A^T$. (See Lemma~\ref{lem:convergeQ}.)
Imagine that we are given the exact matrix $Q$ instead of a noisy approximation.
Notice that $Q$
is a product of {\em three} nonnegative matrices, the first of which
is $p$-separable and the last is the transpose of the first. NMF at
first sight seems too weak to help find such factorizations.
However, if we think of $Q$ as a product of
{\em two} nonnegative matrices, $A$ and $R(\mathcal{T})A^T$, then our NMF
algorithm \cite{AGKM} can at least identify the
anchor words of $A$. As noted above, these 
suffice to  recover $R(\mathcal{T})A^T$, and then (using the anchor
words of $A$ again) all of $A$ as well.
See Section~\ref{subsec:idealanchor}  for details.

Of course, we are not given $Q$ but merely a good approximation to
it.  Now our NMF algorithm allows us to recover ``almost anchor''
words of $A$, and the crux of the proof is Section~\ref{subsec:almostanchor}
showing that these suffice to recover provably good estimates
to $A$ and $WW^T$. This uses (mostly)  bounds from matrix
perturbation theory, and interrelationships of condition numbers
mentioned in Section~\ref{sec:cond}.




For simplicity we assume the following condition on the topic model,  which we will see in
Section~\ref{subsec:freqwlog} can be assumed without loss of generality:

\hspace*{0.25in} (*) {\em  
The number of words, $n$, is at most $4ar/\epsilon$. 
}

\noindent Please see Algorithm~\ref{alg:main}: Main Algorithm for description of the
algorithm. Note that $R$ is our shorthand for $\frac{1}{m}WW^T$, which
as noted converges to
$R(\mathcal{T})$ as the number of documents increases.


\begin{fragment*}[t]
\caption{
\label{alg:main}{\sc Main Algorithm}, \textbf{Output:} $R$ and $A$\vspace*{0.01in}
}

\begin{enumerate} \itemsep 0pt
\small 
\item Query the oracle for $m$ documents, where $$m = \max\left\{O\left(\frac{\log n \cdot a^4r^6}{\epsilon^2p^6\gamma^2 N}\right), O\left(\frac{\log r \cdot a^2 r^4}{\gamma^2}\right) , O\left(\frac{\log r\cdot r^2}{\epsilon^2}\right)\right\}$$
\item Split the words of each document into two halves, and let $\tilde{M}$, $\tilde{M}'$ be the term-by-document matrix with first and second half of words respectively.
\item Compute word-by-word matrix $Q = \frac{4}{N^2m}\tilde{M} \tilde{M}'^T$
\item Apply the ``Robust NMF" algorithm of
  Theorem~\ref{thm:betterseparablenoise}  to $Q$ which returns $r$ words that
  are "almost" the anchor words of $A$. 
\item Use these $r$ words as input to {\sc Recover with Almost Anchor Words} to compute $R = \frac{1}{m}WW^T$ and $A$
\end{enumerate} 

\end{fragment*}

\vspace*{-0.1in}
\subsection{Recover $R$ and $A$ with Anchor Words}
\label{subsec:idealanchor}
We first describe how the recovery procedure works in an ``idealized'' setting (Algorthm~\ref{alg:recoveranchor},{\sc Recover with True Anchor Words}), when we are given the exact value of $ARA^T$ and a set of anchor words -- one for each topic. We can permute the rows of $A$ so that the anchor words are exactly the first $r$ words. Therefore $A^T = (D, U^T)$ where $D$ is a diagonal matrix. Note that $D$ is not necessarily the identity matrix (nor even a scaled copy of the identity matrix), but we do know that the diagonal entries are at least $p$. We apply the same permutation to the rows and columns of $Q$. As shown in Figure~\ref{fig:recover}, if we look at the submatrix formed by the first $r$ rows and $r$ columns, it is exactly $DRD$. Similarly, the submatrix consisting of the first $r$ rows is exactly $DRA^T$. We can use these two matrices to compute $R$ and $A$, in this idealized setting (and we will use the same basic strategy in the general case, but need only be more careful about how we analyze how errors compound in our algorithm).

\begin{fragment*}[t]
\caption{\label{alg:recoveranchor}
{\sc Recover with True Anchor Words} \newline \textbf{Input:} $r$ anchor words, \textbf{Output:} $R$ and $A$}

\begin{enumerate} \itemsep 0pt
\small 
\item Permute the rows and columns of $Q$ so that the anchor words appear in the first $r$ rows and columns
\item Compute $DRA^T \vec{1}$ (which is equal to $DR\vec{1} $)
\item Solve for $\vec{z}$: $DRD\vec{z} = DR\vec{1}$.
\item Output $A^T = ((DRD\mbox{Diag}(z))^{-1}DRA^T)$.
\item Output $R = (\mbox{Diag}(z)DRD\mbox{Diag}(z))$.
\end{enumerate}
\end{fragment*}

Our algorithm has exact knowledge of the matrices $DRD$ and $DRA^T$, and so the main task is to recover the diagonal matrix $D$. Given $D$, we can then compute $A$ and $R$ (for the Dirichlet Allocation we can also compute its parameters - i.e. the $\vec{\alpha}$ so that $R(\alpha) = R$). The key idea to this algorithm is that the row sums of $DR$ and $DRA^T$ are the same, and we can use the row sums of $DR$ to set up a system of linear constraints on the diagonal entries of $D^{-1}$. 

\begin{lemma}\label{lem:recoverrealanchor}
When the matrix $Q$ is exactly equal to $ARA^T$ and we know the set of anchor words, {\sc Recover with True Anchor Words} outputs $A$ and $R$ correctly.
\end{lemma}

\begin{proof}
The Lemma is straight forward from Figure~\ref{fig:recover} and the procedure. By Figure~\ref{fig:recover} we can find the exact value of $DRA^T$ and $DRD$ in the matrix $Q$. Step 2 of recover computes $DR\vec{1}$ by computing $DRA^T\vec{1}$. The two vectors are equal because $A$ is the topic-term matrix and its columns sum up to 1, in particular $A^T\vec{1} = \vec{1}$.

In Step 3, since $R$ is invertible by Lemma~\ref{lem:convergeR}, $D$ is a diagonal matrix with entries at least $p$, the matrix $DRD$ is also invertible. Therefore there is a unique solution $\vec{z} = (DRD)^{-1}DR\vec{1} = D^{-1}\vec{1}$. Also $D \vec{z} = \vec{1}$ and hence $D \mbox{Diag}(z) = I$. 
Finally, using the fact that $D\mbox{Diag}(z) = I$, the output in step 4 is just $(DR)^{-1}DRA^T = A^T$, and the output in step 5 is equal to $R$.
\end{proof}

\begin{figure}
  \center
  \includegraphics[width = 6in]{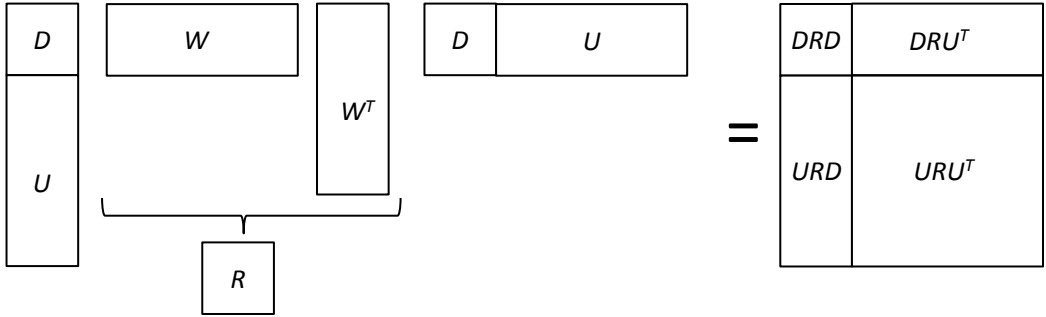}\\
\caption{The matrix $Q$}
  \label{fig:recover}
\end{figure}

\vspace*{-0.2in}
\subsection{Recover $R$ and $A$ with Almost Anchor Words}\label{subsec:almostanchor}

What if we are not given the exact anchor words, but are given words that are ``close'' to anchor words? As we noted, in general we cannot hope to recover the true anchor words, but even a good approximation will be enough to recover $R$ and $A$. 

When we restrict $A$ to the rows corresponding to ``almost'' anchor words, the submatrix will not be diagonal. However, it will be close to a diagonal in the sense that the submatrix will be a diagonal matrix $D$ multiplied by $E$, and $E$ is close to the identity matrix (and the diagonal entries of $D$ are at least $\Omega(p)$). Here we analyze the same procedure as above and show that it still recovers $A$ and $R$ (approximately) even when given ``almost'' anchor words instead of true anchor words.  For clarity we state the procedure again in Algorithm~\ref{alg:recoveralmostanchor}: {\sc Recover with Almost Anchor Words}. The guarantees at each step are different than before, but the implementation of the procedure is the same. Notice that here we permute the rows of $A$ (and hence the rows and columns of $Q$) so that the ``almost'' anchor words returned by Theorem~\ref{thm:separablenoise} appear first and the submatrix $A$ on these rows is equal to $DE$.

Here, we still assume that the matrix $Q$ is exactly equal to $ARA^T$ and hence the first $r$ rows of $Q$ form the submatrix $DERA^T$ and the first $r$ rows and columns are $DERE^TD$. The complication here is that $\mbox{Diag}(z)$ is not necessarily equal to $D^{-1}$, since the matrix $E$ is not necessarily the identity. However, we can show that $\mbox{Diag}(z)$ is "close" to $D^{-1}$ if $E$ is suitably close to the identity matrix -- i.e. given good enough proxies for the anchor words, we can bound the error of the above recovery procedure. We write $E = I+Z$. Intuitively when $Z$ has only small entries $E$ should behave like the identity matrix. In particular, $E^{-1}$ should have only small off-diagonal entries. We make this precise through the following lemmas:

\begin{lemma} \label{lem:inverseE1}
Let $E = I+Z$ and $\sum_{i, j} |Z_{i,j}| = \epsilon < 1/2$, then $E^{-1} \vec{1}$ is a vector with entries in the range $[1-2\epsilon, 1+2\epsilon]$.
\end{lemma}

\begin{proof}
$E$ is clearly invertible because the spectral norm of $Z$ is at most $1/2$.
Let $\vec{b} = E^{-1}\vec{1}$. Since $E = I+Z$ we multiply $E$ on both sides to get $\vec{b}+Z\vec{b} = \vec{1}$. Let $b_{max}$ be the largest absolute value of any entry of $b$ ($b_{max} = \max |b_i|$). Consider the entry $i$ where $b_{max}$ is achieved, we know $b_{max} = |b_i| \le 1 + |(Zb)_i| \le 1 + \sum_{j} |Z_{i,j}||b_j| \le 1+\epsilon b_{max}.$
Thus $b_{max}\le 1/(1-\epsilon) \le 2$. Now all the entries in $Z\vec{b}$ are within $2\epsilon$ in absolute value, and we know that $\vec{b} = \vec{1}+Z\vec{b}$. Hence all the entries of $b$ are in the range $[1-2\epsilon, 1+2\epsilon]$, as desired. 
\end{proof}

\begin{lemma} \label{lem:inverseE2}
Let $E = I+Z$ and $\sum_{i, j} |Z_{i,j}| = \epsilon < 1/2$, then the columns of $E^{-1} - I$ have $\ell_1$ norm at most $2\epsilon$.
\end{lemma}

\begin{proof}
Without loss of generality, we can consider just the first column of $E^{-1}-I$, which is equal to $(E^{-1}-I)\vec{e_1}$, where $\vec{e_1}$ is the indicator vector that is one on the first coordinate and zero elsewhere.

The approach is similar to that in Lemma~\ref{lem:inverseE1}. Let $\vec{b} =  (E^{-1}-I)\vec{e_1}$. Left multiply by $E = (I + Z)$ and we obtain $\vec{b} + Z\vec{b}=  - Z \vec{e_1}$. Hence $\vec{b} = - Z (\vec{b} +  \vec{e_1}) $.  Let $b_{max}$ be the largest absolute value of entries of $\vec{b}$ ($b_{max} = \max |b_i|$). Let $i$ be the entry in which $b_{max}$ is achieved. Then
$$b_{max} = |b_i| \le |(Z\vec{b} )_i| + |(Z\vec{e_1})_i | \le  \epsilon b_{max} + \epsilon$$
Therefore $b_{max} \le \epsilon/(1-\epsilon) \le 2\epsilon$. Further, the $\| \vec{b}\|_1 \le \|Z\vec{e_1}\|_1+\|Z\vec{b}\|_1\le \epsilon+2\epsilon^2 \le 2\epsilon$. 
\end{proof}

Now we are ready to show that the procedure {\sc Recover with Almost Anchor Words} succeeds when given "almost" anchor words:

\begin{fragment*}[t]
\caption{\label{alg:recoveralmostanchor}
{\sc Recover with Almost Anchor Words} \newline \textbf{Input:} $r$ "almost" anchor words, \textbf{Output:} $R$ and $A$}
\begin{enumerate} \itemsep 0pt
\small 
\item Permute the rows and columns of $Q$ so that the "almost" anchor words appear in the first $r$ rows and columns.
\item Compute $DERA^T \vec{1}$ (which is equal to $DER\vec{1}$)
\item Solve for $\vec{z}$: $DERE^TD\vec{z} = DER\vec{1}$.
\item Output $A^T = ((DERE^TD\mbox{Diag}(z))^{-1}DERA^T)$.
\item Output $R = (\mbox{Diag}(z)DERE^TD\mbox{Diag}(z))$.
\end{enumerate}

\end{fragment*}

\begin{lemma}
\label{lem:recoveralmostanchor}
When the matrix $Q$ is exactly equal to $ARA^T$, the matrix $A$ restricted to almost anchor words is $DE$ where $E-I$ has $\ell_1$ norm $\epsilon<1/10$ when viewed as a vector, procedure {\sc Recover with Almost Anchor Words} outputs $A$ such that each column of $A$ has $\ell_1$ error at most $6\epsilon$. The matrix $R$ has additive error $Z_R$ whose $\ell_1$ norm when viewed as a vector is at most $8\epsilon$.
\end{lemma}

\begin{proof}
Since $Q$ is exactly $ARA^T$, our algorithm is given $DERA^T$ and $DERE^TD$ with no error. In Step 3, since $D$, $E$ and $R$ are all invertible, we have $$\vec{z} = (DERE^TD)^{-1} DER \vec{1} = D^{-1}(E^T)^{-1} \vec{1}$$
 Ideally we would want $\mbox{Diag}(z) = D^{-1}$, and indeed $D\mbox{Diag}(z) = \mbox{Diag}((E^T)^{-1} \vec{1})$. From Lemma~\ref{lem:inverseE1}, the vector $(E^T)^{-1} \vec{1}$ has entries in the range $[1-2\epsilon, 1+2\epsilon]$, thus each entry of $\mbox{Diag}(z)$ is within a $(1 \pm 2 \epsilon)$ multiplicative factor from the corresponding entry in $D^{-1}$.

Consider the output in Step 4. Since $D$, $E$, $R$ are invertible, the first output is $$(DERE^TD\mbox{Diag(z)})^{-1}DERA^T = (D\mbox{Diag}(z))^{-1} (E^T)^{-1} A^T$$ 
Our goal is to bound the $\ell_1$ error of the columns of the output compared to the corresponding columns of $A$. Notice that it is sufficient to show that the $j^{th}$ row of $(D\mbox{Diag}(z))^{-1} (E^T)^{-1}$ is close (in $\ell_1$ distance) to the indicator vector $\vec{e_j}^T$. 

\begin{claim}
For each $j$, $ \| \vec{e_j}^T (D\mbox{Diag}(z))^{-1} (E^T)^{-1} - \vec{e_j}^T \|_1 \leq 5 \epsilon$
\end{claim}

\begin{proof}
Again, without loss of generality we can consider just the first row. From Lemma~\ref{lem:inverseE2} $\vec{e_1}^T (E^T)^{-1}$ has $\ell_1$ distance at most $2\epsilon$ to $\vec{e_1}^T$. $(D\mbox{Diag}(z))^{-1}$ has entries in the range $[1-3\epsilon, 1+3\epsilon]$. And so $$ \|\vec{e_1}^T (D\mbox{Diag}(z))^{-1} (E^T)^{-1}  - \vec{e_1}^T \|_1 \leq \|  \vec{e_1}^T (D\mbox{Diag}(z))^{-1} (E^T)^{-1} - \vec{e_1}^T  (E^T)^{-1}  \|_1 + \|\vec{e_1}^T  (E^T)^{-1}  - \vec{e_1}^T \|_1$$
The last term can be  bounded by $2 \epsilon$. Consider the first term on the right hand side: The vector $ \vec{e_1}^T (D\mbox{Diag}(z))^{-1}  - \vec{e_1}^T$ has one non-zero entry (the first one) whose absolute value is at most $3 \epsilon$. Hence, from Lemma~\ref{lem:inverseE2} the first term can be bounded by $6 \epsilon^2 \leq 3 \epsilon$, and this implies the claim. 
\end{proof}

The first row of $(D\mbox{Diag}(z))^{-1} (E^T)^{-1} A^T$ is $A_1 + z^T A$ where $z$ is a vector with $\ell_1$ norm at most $5\epsilon$. So every column of $A$ is recovered with $\ell_1$ error at most $6\epsilon$. 

Consider the second output of the algorithm. The output is $\mbox{Diag}(z)DE R E^T D\mbox{Diag}(z)$ and we can write $\mbox{Diag}(z)D = I+Z_1$ and $E = I+Z_2$.  The leading error are $Z_1R+Z_2R+RZ_1+RZ_2$ and hence the $\ell_1$ norm of the leading error term (when treated as a vector) is at most $6\epsilon$ and other terms are of order $\epsilon^2$ and can safely be bounded by $2\epsilon$ for suitably small $\epsilon$).
\end{proof}

Finally we consider the general case (in which there is additive noise in Step 1): we are not given $ARA^T$ exactly. We are given $Q$ which is close to $ARA^T$ (by Lemma~\ref{lem:convergeQ}). We will bound the accumulation of this last type of error. Suppose in Step $1$ of RECOVER we obtain $DERA^T+U$ and $DERE^TD+V$ and furthermore the entries of $U$ and $U\vec{1}$ have absolute value at most $\epsilon_1$ and the matrix $V$ has $\ell_1$ norm $\epsilon_2$ when viewed as a vector.

\begin{lemma}
\label{lem:recoverwitherror}
 If $\epsilon, \epsilon_1, \epsilon_2$ are sufficiently small, RECOVER outputs $A$ such that each entry of $A$ has additive error at most $O(\epsilon+(ra\epsilon_2/p^3 +\epsilon_1 r /p^2)/\gamma)$. Also the matrix $R$ has additive error $Z_R$ whose $\ell_1$ norm when viewed as a vector is at most $O(\epsilon+(ra\epsilon_2/p^3 +\epsilon_1 r/p^2)/\gamma)$.
\end{lemma}

The main idea of the proof is to write $DERE^TD+V$ as $DER(E^T+V')D$. In this way the error $V$ can be translated to an error $V'$ on $E$ and Lemma~\ref{lem:recoveralmostanchor} can be applied. The error $U$ can be handled similarly.

\begin{proof}
We shall follow the proof of Lemma~\ref{lem:recoveralmostanchor}. First can express the error term $V$ instead as $V = (DER) V' (D)$. This is always possible because all of $D$, $E$, $R$ are invertible. Moreover, the $\ell_1$ norm of $V'$ when viewed as a vector is at most $8ra\epsilon_2/\gamma p^3$, because this norm will grow by a factor of at most $1/p$ when multiplied by $D^{-1}$, a factor of at most 2 when multiplied by $E^{-1}$ and at most $ra/\Gamma(R)$ when multiplied by $R^{-1}$. The bound of $\Gamma(R)$ comes from Lemma~\ref{lem:convergeR}, we lose an extra $ra$ because $R$ may not have rows sum up to 1.

Hence $DERE^TD + V  = DER(E^T+V') D$ and the additive error for $DERE^TD$ can be transformed into error in $E$, and we will be able to apply the analysis in Lemma~\ref{lem:recoveralmostanchor}.

Similarly, we can express the error term $U$ as $U = DER U'$. Entries of $U'$ have absolute value at most $8\epsilon_1 r/\gamma p^2$. The right hand side of the equation in step 3 is equal to $DER \vec{1}+U\vec{1}$ so the error is at most $\epsilon_1$ per entry. 
Following the proof of Lemma~\ref{lem:recoveralmostanchor}, we know $\mbox{Diag}(z)D$ has diagonal entries within $1\pm \left(2\epsilon + 16\epsilon_2/\gamma p^3 + 2\epsilon_1 \right)$.

Now we consider the output. The output for $A^T$ is equal to 
$$(DER(E^T+V')D\mbox{Diag}(z))^{-1}DER (A^T+U') = (D\mbox{Diag}(z))^{-1} (E^T+V')^{-1} (A^T+U').$$
Here we know $(E^T+V')^{-1}-I$ has $\ell_1$ norm at most $O(\epsilon+ra \epsilon_2/\gamma p^3)$ per row, $(D\mbox{Diag}(z))$ is a diagonal matrix with entries in $1\pm O(\epsilon+ra \epsilon_2/\gamma p^3 +\epsilon_1)$, entries of $U'$ has absolute value $O(\epsilon_1 r/\gamma p^2)$. Following the proof of Lemma~\ref{lem:recoveralmostanchor} the final entry-wise error of $A$ is roughly the sum of these three errors, and is bounded by $O(\epsilon+(ra \epsilon_2/p^3 +\epsilon_1 r/p^2)/\gamma)$ (Notice that Lemma~\ref{lem:recoveralmostanchor} gives bound for $\ell_1$ norm of rows, which is stronger. Here we switched to entry-wise error because the entries of $U$ are bounded while the $\ell_1$ norm of $U$ might be large).

Similarly, the output of $R$ is equal to $\mbox{Diag}(z) (DERE^TD + V) \mbox{Diag}(z)$. Again we write $\mbox{Diag}(z)D = I+Z_1$ and $E = I+Z_2$. The extra term $\mbox{Diag}(z) V \mbox{Diag}(z)$ is small because the entries of $z$ are at most to $2/p$ (otherwise $\mbox{Diag}(z)D$ won't be close to identity). The error can be bounded by $O(\epsilon+(ra \epsilon_2/p^3 +\epsilon_1 r/p^2)/\gamma)$.
\end{proof}

Now in order to prove our main theorem we just need to show that when number of documents is large enough, the matrix $Q$ is close to the $ARA^T$, and plug the error bounds into Lemma~\ref{lem:recoverwitherror}. 

\subsection{Error Bounds for $Q$}

Here we show that the matrix $Q$ indeed converges to $\frac{1}{m}AWW^TA^T = ARA^T$ when $m$ is large enough. 

\begin{lemma}
\label{lem:convergeQ}
When $m > \frac{50 \log n}{N\epsilon_Q^2}$, with high probability all entries of $Q - \frac{1}{m}A WW^TA^T$ have absolute value at most $\epsilon_Q$. Further, the $\ell_1$ norm of rows of $Q$ are also $\epsilon_Q$ close to the $\ell_1$ norm of the corresponding row in $\frac{1}{m} AWW^TA^T$. 
\end{lemma}

\begin{proof}
We shall first show that the expectation of $Q$ is equal to
$ARA^T$ where $R$ is $\frac{1}{m}WW^T$. Then by concentration bounds we show that entries of $Q$ are close to their expectations. Notice that we can also hope to show that $Q$ converges to $AR(\mathcal{T})A^T$. However in that case we will not be able to get the inverse polynomial relationship with $N$ (indeed, even if $N$ goes to infinity it is impossible to learn $R(\mathcal{T})$ with only one document). Replacing $R(\mathcal{T})$ with the empirical $R$ allows our algorithm to perform better when the number of words per document is larger.

To show the expectation is correct we observe that conditioned on $W$, the entries of two matrices $\tilde{M}$ and $\tilde{M}'$ are independent. Their expectations are both $\frac{N}{2}AW$. Therefore,

\[
\E[Q] = \frac{4}{mN^2}\E[\tilde{M} \tilde{M}'^T] = \frac{1}{m} \left(\frac{2}{N} \E[\tilde{M}]\right) \left(\frac{2}{N} \E[\tilde{M}'^T]\right)= \frac{1}{m}  A W W^T A^T = A RA^T. 
\]

We still need to show that $Q$ is close to this expectation. This is not surprising because $Q$ is the average of $m$ independent samples (of $\frac{4}{N^2} \tilde{M}_i\tilde{M}_i'$). Further, the variance of each entry in $\frac{4}{N^2}\tilde{M}_i \tilde{M}_i'^T$ can be bounded because $\tilde{M}$ and $\tilde{M}'$ also come from independent samples. For any $i$, $j_1$, $j_2$, let $v = AW_i$ be the probability distribution that $\tilde{M}_i$ and $\tilde{M}_i'$ are sampled from, then $\tilde{M}_i(j_1)$ is distributed as $Binomial(N/2, v(j_1))$ and $\tilde{M}'_i(j_2)$ is distributed as $Binomial(N/2, v(j_2))$. The variance of these two variables are less than $N/8$ no matter what $v$ is by the properties of binomial distribution. Conditioned on the vector $v$ these two variables are independent, thus the variance of their product is at most $\Var\tilde{M}_i(j_1) \E \tilde{M}'_i(j_2)^2+\E\tilde{M}_i(j_1)^2 \Var \tilde{M}'_i(j_2)+\Var\tilde{M}_i(j_1) \Var \tilde{M}'_i(j_2) \le N^3/4+N^2/64$. The variance of any entry in $\frac{4}{N^2}\tilde{M}_i \tilde{M}_i'^T$ is at most $4/N+1/16N^2 = O(1/N)$. Higher moments can be bounded similarly and they satisfy the assumptions of Bernstein inequalities. Thus 
by Bernstein inequalities the probability that any entry is more than $\epsilon_Q$ away from its true value is much smaller than $1/n^2$.


The further part follows from the observation that the $\ell_1$ norm of a row in $Q$ is proportional to the number of appearances of the word. As long as the number of appearances concentrates the error in $\ell_1$ norm must be small. The words are all independent (conditioned on $W$) so this is just a direct application of Chernoff bounds.
\end{proof}



\subsection{Proving the Main Theorem}

We are now ready to prove Theorem~\ref{thm:main}:

\vspace{0.5pc}

\begin{proof}
By Lemma~\ref{lem:convergeQ} we know when we have at least $50\log n/N\epsilon_Q^2$ documents, $Q$ is entry wise close to $ARA^T$. In this case error per row for Theorem~\ref{thm:separablenoise} is at most $\epsilon_Q \cdot  O(a^2r^2/p^2)$ because in this step we can assume that there are at most $4ar/p$ words (see Section~\ref{subsec:freqwlog}) and to normalize the row we need a multiplicative factor of at most $10ar/p$ (we shall only consider rows with $\ell_1$ norm at least $p/10ar$, with high probability all the anchor words are in these rows). The $\gamma$ parameter for Theorem~\ref{thm:betterseparablenoise} is $p\gamma/4$ by Lemma~\ref{lem:convergeR}. Thus the almost anchor words found by the algorithm has weight at least $1 - O(\epsilon_Q a^2r^2/ \gamma p^3)$ on diagonals. The error for $DERE^TD$ is at most $\epsilon_Q r^2$, the error for any entry of $DERA^T$ and $DERA^T\vec{1}$ is at most $O(\epsilon_Q)$. Therefore by Lemma~\ref{lem:recoverwitherror} the entry-wise error for $A$ is at most $O(\epsilon_Q a^2r^3/\gamma p^3)$. 

When $\epsilon_Q < \epsilon p^3\gamma/a^2r^3$ the error is bounded by $\epsilon$ 
. In this case we need 

 $$m = \max\left\{O\left(\frac{\log n \cdot a^4r^6}{\epsilon^2p^6\gamma^2 N}\right), O\left(\frac{\log r \cdot a^2 r^4}{\gamma^2}\right)\right\}.$$

The latter constraint comes from Lemma~\ref{lem:convergeR}.

To get within $\epsilon$ additive error for the parameter $\alpha$, we further need $R$ to be close enough to the variance-covariance matrix of the document-topic distribution, which means $m$ is at least 

 $$m = \max\left\{O\left(\frac{\log n \cdot a^4r^6}{\epsilon^2p^6\gamma^2 N}\right), O\left(\frac{\log r \cdot a^2 r^4}{\gamma^2}\right) , O\left(\frac{\log r\cdot r^2}{\epsilon^2}\right)\right\}.$$

\end{proof}

\subsection{Reducing Dictionary Size}
\label{subsec:freqwlog}
Above we assumed that the number of distinct words is small. Here, we give a simple gadget that shows in the general case we
can assume that this is the case at the loss of an additional additive $\epsilon$ in our accuracy: 

\begin{lemma}
The general case can be reduced to an instance in which there are at most $4ar/\epsilon$ words all of which (with at most one exception) occur with probability at least $\epsilon/4ar$.
\end{lemma}



\begin{proof} 
In fact, we can collect all words that occur infrequently and ``merge'' all of these words into a aggregate word that we will call the {\em runoff word}. To this end, we call a word large if it appears more than $\epsilon mN/3ar$ times in $m = \frac{100ar \log n}{N\epsilon}$ documents, and otherwise we call it small. Indeed, with high probability all large words are words that occur with probability at least $\epsilon/4ar$ in our model. Also, all words that has a entry larger than $\epsilon$ in the corresponding row of $A$ will appear with at least $\epsilon/ar$ probability, and is thus a large word with high probability. We can merge all small words (i.e. rename all of these words to a single, new word). Hence we can apply the above algorithm (which assumed that there are not too many distinct words). After we get a result with the modified documents we can ignore the {\em runoff words} and assign $0$ weight for all the small words. The result will still be correct up to $\epsilon$ additive error.
\end{proof}

\vspace*{-0.1in}
\section{The Dirichlet Subcase}
\label{sec:dirichlet}
Here we demonstrate that the parameters of a Dirichlet distribution can be (robustly) recovered from just the covariance matrix $R(\calT)$.
Hence an immediate corollary is that our main learning algorithm can recover both the topic matrix $A$ and the distribution that generates columns of $W$
in a {\em Latent Dirichlet Allocation} (LDA) Model \cite{LDA}, provided that $A$ is separable. We believe that this algorithm may be of practical use, and provides the
first alternative to local search and (unproven) approximation procedures for this inference problem \cite{WJ}, \cite{DLR}, \cite{LDA}.



The Dirichlet distribution is parametrized by a vector $\alpha$ of positive reals is a natural  family of continuous multivariate probability distributions. 
The support of the Dirichlet Distribution is the unit simplex whose dimension is the same as the dimension of $\alpha$. Let $\alpha$ be a $r$ dimensional vector. Then for a vector $\theta \in \R^r$ in the $r$ dimensional simplex, its probability density is given by

\[
Pr[\theta|\alpha] = \frac{\Gamma(\sum_{i=1}^r \alpha_i)}{\prod_{i=1}^r\Gamma(\alpha_i)} \prod_{i=1}^r \theta_i^{\alpha_i-1},
\]

\noindent where $\Gamma$ is the Gamma function. In particular, when all the $\alpha_i$'s are equal to one, the Dirichlet Distribution is just the uniform random distribution over the probability simplex.

The expectation and variance of $\theta_i$'s are easy to compute given the parameters $\alpha$. We denote $\alpha_0 = \norm{\alpha}_1 = \sum_{i=1}^r \alpha_i$, then the ratio $\alpha_i/\alpha_0$ should be interpreted as the ``size'' of the $i$-th variable $\theta_i$, and $\alpha_0$ shows whether the distributions is concentrated in the interior (when $\alpha_0$ is large) or near the boundary (when $\alpha_0$ is small). The first two moments of Dirichlet Distribution is listed as below:

\[
\E[\theta_i ] = \frac{\alpha_i}{\alpha_0}.
\]
\[
\E[\theta_i\theta_j] = \left\{ \begin{array}{cl} \frac{\alpha_i\alpha_j}{\alpha_0(\alpha_0+1)} & \mbox{when } i\neq j \\ \frac{\alpha_i(\alpha_i+1)}{\alpha_0(\alpha_0+1)} & \mbox{when } i = j \end{array}\right. .
\]

Suppose the Dirichlet distribution has $\max \alpha_i/\min \alpha_i = a$ and the sum of parameters is $\alpha_0$; 
we give an algorithm that computes close estimates to the vector of parameters $\alpha$ given a sufficiently close estimate to the co-variance matrix $R(\calT)$ (Theorem~\ref{thm:dirrecov}). Combining this with Theorem~\ref{thm:main}, we obtain the following corollary: 

\begin{theorem}
\label{thm:dirichlet}
There is an algorithm that learns the topic matrix $A$ with high probability up to an additive error of $\epsilon$ from at most $$m = \max\left\{O\left(\frac{\log n \cdot a^6r^8(\alpha_0+1)^4}{\epsilon^2p^6 N}\right), O\left(\frac{\log r\cdot a^2r^4(\alpha_0+1)^2}{\epsilon^2}\right)\right\}$$ documents sampled from the LDA model and runs in time polynomial in $n$, $m$. Furthermore, we also recover the parameters of the Dirichlet distribution to within an additive $\epsilon$.
\end{theorem}

Our main goal in this section is to bound the $\ell_1$-condition number of the Dirichlet distribution (Section~\ref{subsec:condfordirichlet}), and using this we show how to recover the parameters of the distribution from its covariance matrix (Section~\ref{subsec:recoverdirichlet}).

\subsection{Condition Number of a Dirichlet Distribution}\label{sec:conddir}
\label{subsec:condfordirichlet}

  There is a well-known meta-principle that if a matrix $W$  is chosen by picking its
columns independently from a fairly diffuse distribution, then it will
be far from low rank. However, our analysis will require us to prove an explicit lower bound on
$\Gamma(R(\calT))$. We now prove such a bound when the columns of $W$ are chosen from a Dirichlet
distribution with parameter vector $\alpha$. We note that it is easy to establish such bounds for other types
of distributions as well. Recall that we defined $R(\calT)$ in Section~\ref{sec:back}, and here we will abuse
notation and throughout this section we will denote by $R(\alpha)$ the matrix $R(\calT)$ where $\calT$ is a Dirichlet distribution with parameter $\alpha$. 



Let $\alpha_0 = \sum_{i = 1}^r \alpha_i$. The mean, variance and co-variance for a Dirichlet distribution are well-known, from which we observe that $R(\alpha)_{i,j}$ is equal to $\frac{\alpha_i \alpha_j}{\alpha_0 (\alpha_0 + 1)}$ when $i \neq j$ and is equal to $\frac{\alpha_i (\alpha_i + 1)}{\alpha_0 (\alpha_0 + 1)}$ when $i = j$. 

\begin{lemma}\label{lem:dirichcondn}
The $\ell_1$ condition number of $R(\alpha)$ is at least $\frac{1}{2(\alpha_0+1)}$. 
\end{lemma}


\begin{proof}
As the entries $R(\alpha)_{i,j}$ is $\frac{\alpha_i \alpha_j}{\alpha_0 (\alpha_0 + 1)}$ when $i \neq j$ and $\frac{\alpha_i (\alpha_i + 1)}{\alpha_0 (\alpha_0 + 1)}$ when $i = j$, after normalization $R(\alpha)$ is just the matrix $D' = \frac{1}{\alpha_0+1}\left(\alpha \times (1,1, ..., 1) + I\right)$ where $\times$ is outer product and $I$ is the identity matrix. 

Let $x$ be a vector such that $|x|_1 = 1$ and $|D' x|_1$ achieves the minimum in $\Gamma(R(\alpha))$ and let $I = \{i | x_i \geq 0\}$ and let $J = \bar{I}$ be the complement. We can assume without loss of generality that $\sum_{i \in I} x_i \geq | \sum_{i \in J} x_i |$ (otherwise just take $-x$ instead). The product $D'x$ is $ \frac{\sum x_i}{\alpha_0+1} \alpha + \frac{1}{\alpha_0+1} x$. The first term is a nonnegative vector and hence for each $i \in I$, $(D' x)^i \geq 0$. This implies that $$| D' x |_1 \geq \frac{1}{\alpha_0+1} \sum_{i \in I} x_i \ge \frac{1}{2(\alpha_0+1)}.$$
\end{proof}

\subsection{Recovering the Parameters of a Dirichlet Distribution}\label{sec:recovdir}
\label{subsec:recoverdirichlet}
When the variance covariance matrix $R(\alpha)$ is recovered with error $\epsilon_R$ in $\ell_1$ norm when viewed as a vector, we can use Algorithm~\ref{alg:dirichlet}: {\sc Dirichlet} to compute the parameters for the Dirichlet.

\begin{fragment*}[t]
\caption{\label{alg:dirichlet}
{\sc Dirichlet($R$)},  \textbf{Input:} $R$, \textbf{Output:} $\alpha$ (vector of parameters)}

\begin{enumerate} \itemsep 0pt
\small 
\item Set $\alpha / \alpha_0 = R \vec{1}$.
\item Let $i$ be the row with smallest $\ell_1$ norm, let $u = R_{i.i}$ and $v = \alpha_i/\alpha_0$.
\item Set $\alpha_0 = \frac{1 - u/v} {u/v - v}$.
\item Output $\alpha = \alpha_0 \cdot \left(\alpha/\alpha_0\right)$.
\end{enumerate}
\end{fragment*}

\begin{theorem}\label{thm:dirrecov}
When the variance covariance matrix $R(\alpha)$ is recovered with error $\epsilon_R$ in $\ell_1$ norm when viewed as a vector, the procedure {\sc Dirichlet($R$)} learns the parameter of the Dirichlet distribution with error at most $O(ar(\alpha_0+1) \epsilon_R)$.
\end{theorem}

\begin{proof}
The $\alpha_i/\alpha_0$'s all have error at most $\epsilon_R$. The value $u$ is $\frac{\alpha_i}{\alpha_0}\frac{\alpha_i+1}{\alpha_0+1}\pm \epsilon_R$ and the value $v$ is $\alpha_i/\alpha_0\pm \epsilon_R$. Since $v \ge 1/ar$  we know the error for $u/v$ is at most $2 a r\epsilon_R$. Finally we need to bound the denominator $\frac{\alpha_i+1}{\alpha_0+1} - \frac{\alpha_i}{\alpha_0} > \frac{1}{2(\alpha_0+1)}$ (since $\frac{\alpha_i}{\alpha_0} \le 1/r\le 1/2$). Thus the final error is at most $5ar(\alpha_0+1) \epsilon_R$.
\end{proof}

\section{Obtaining Almost Anchor Words}
\label{sec:betternmf}
In this section, we prove Theorem~\ref{thm:betterseparablenoise}, which we restate here:

\begin{theorem}
Suppose $M=AW$ where $W$ and $M$ are normalized to have rows sum up to 1,  
$A$ is separable and $W$ is $\gamma$-robustly simplicial. When $\epsilon < \gamma/100$  there is a polynomial time algorithm that given $\tilde{M}$ such that for all rows $\|\tilde{M}^i - M^i\|_1<\epsilon$, finds $r$ row (almost anchor words) in $\tilde{M}$. The $i$-th almost anchor word corresponds to a row in $M$ that  can be represented as $(1-O(\epsilon/\gamma))W^i + O(\epsilon/\gamma) W^{-i}$. Here $W^{-i}$ is a vector in the convex hull of other rows in $W$ with unit length in $\ell_1$ norm.
\end{theorem}

 The major weakness of the algorithm in \cite{AGKM} is that it only considers the $\ell_1$ norm. However in an $r$ dimensional simplex there is another natural measure of distance more suited to our purposes: since each point is a unique convex combination of the vertices, we can view this convex combination as a probability distribution and use statistical distance (on this representation) as a norm for points inside the simplex. We will in fact need a slight modification to this norm, since we would like it to extend to points outside the simplex too: 
 


\begin{definition} [$(\delta, \epsilon)$-close] A point $M'^j$ is $(\delta, \epsilon)$-close to $M'^i$ if and only if
$$
\min_{c_k \ge 0, \sum_{k=1}^n c_k = 1, c_j \ge 1-\delta} \| M'^i - \sum_{k=1}^n c_k M'^k\|_1 \le \epsilon.
$$
\end{definition}

Intuitively think of point $M'^j$ is $(\delta, \epsilon)$-close to $M'^i$ if $M'^i$ is $\epsilon$ close in $\ell_1$ distance to some point $Q$, where $Q$ is a convex combination of the rows of $M'$ that places at least $1-\delta$ weight on $M'^j$. Notice that this definition is not a distance since it is not symmetric, but we will abuse notation and nevertheless call it a distance function. 
We remark that this distance is easy to compute: To check whether $M'^j$ is $(\delta, \epsilon)$-close to $M'^i$ we just need to solve a linear program that minimizes the $\ell_1$ distance when the $c$ vector is a probability distribution with at least $1-\delta$ weight on $j$ (the constraints on $c$ are clearly all linear). 

We also consider all points that a row $M'^j$ is close to, this is called the neighborhood of $M'^j$. 

\begin{definition}[$(\delta,\epsilon)$-neighborhood] The $(\delta,\epsilon)$-neighborhood of $M'^j$ are the rows $M'^i$ such that $M'^j$ is $(\delta,\epsilon)$-close to $M'^i$.
\end{definition}

 For each point $M'^j$, we know its original (unperturbed) point $M_j$ is in a convex combination of $W^i$'s: $M^j = \sum_{i=1}^r A_{j,i} W^i$. Separability implies that for any column index $i$ there is a  row $f(i)$ in $A$ whose only nonzero entry is in the $i^{th}$ column.
Then $M^{f(i)} = W^i$ and consequently $\|M'^{f(i)} - W^i\|_1 < \epsilon$. Let us call these rows $M'^{f(i)}$ for all $i$ the {\em canonical rows}. 
From the above description the following claim is clear. 

\begin{claim}
\label{claim:repsentationwitherror}
Every row $M'^j$ has $\ell_1$-distance at most $2\epsilon$ to the convex hull of canonical rows.
\end{claim}
\begin{proof} We have:
$$\|M'^j - \sum_{k=1}^r A_{j,k} M'^{f(k)} \|_1 \le  \|M'^j-M^j \|_1  + \| M^j - \sum_{k=1}^r A_{j,k} M^{f(k)}\|_1  +\| \sum_{k=1}^r A_{j,k} (M^{f(k)} - M'^{f(k)})\|_1$$
and we can bound the right hand side by $2 \epsilon$.
\end{proof} 

The algorithm will distinguish rows that are close to vertices and rows that are far by testing whether each row is close (in $\ell_1$ norm) to the convex hull of rows outside its neighborhood. In particular, we define a robust loner as:

\begin{definition}[robust loner]
We call a row $M'^j$ a robust-loner if it has $\ell_1$ distance at most $2\epsilon$ to the convex hull of rows that are outside its $(6\epsilon/\gamma, 2\epsilon)$ neighborhood. 
\end{definition}

Our goal is to show that a row is a robust loner if and only if it is close to some row in $W$. The following lemma establishes one direction:


\begin{lemma}
\label{lem:farfromcanonical}
If $A_{j,t}$ is smaller than $1 - 10\epsilon/\gamma$, the point $M'^j$ cannot be $(6\epsilon/\gamma, 2\epsilon)$-close to the canonical row that corresponds to $W^t$.
\end{lemma}

\begin{proof}
Assume towards contradiction that $M'^j$ is $(6\epsilon/\gamma, 2\epsilon)$-close to the canonical row $M'^i$ which is a perturbation of $W^t$. By definition there must be probability distribution $c\in \R^n$ over the rows such that $c_j \ge 1-6\epsilon/\gamma$, and $\|M'^i - \sum_{k=1}^n c_k M'^k\|_1 \le 2\epsilon$. Now we instead consider the unperturbed matrix $M$, since every row of $M'$ is $\epsilon$ close (in $\ell_1$ norm) to $M$ we know $\| M^i - \sum_{k=1}^n c_k M^k \|_1 \le 4\epsilon$. Now we represent $M^i$ and $\sum_{k=1}^n c_k M^k$ as convex combinations of rows of $W$ and consider the coefficient on $W^t$. Clearly $M^i = W^t$ so the coefficient is 1. But for $\sum_{k=1}^n c_k M^k$, since $c_j\ge 1-6\epsilon/\gamma$ and the coefficient $A_{j,t} \le 1-10\epsilon/\gamma$, we know the coefficient of $W^t$ in the sum must be strictly smaller than $1-10\epsilon/\gamma+6\epsilon/\gamma = 1-4\epsilon/\gamma$. By the robustly simplicial assumption $M^i$ and $\sum_{k=1}^n c_k M^k$ must be more than $4\epsilon/\gamma \cdot \gamma = 4\epsilon$ apart in $\ell_1$ norm, which contradicts our assumption. 
\end{proof}

As a corollary:

\begin{corollary}
If $A_{j,t}$ is smaller than $1-10\epsilon/\gamma$ for all $t$, the row $M'^j$ cannot be a robust loner.
\end{corollary}

\begin{proof}
By the above lemma, we know the canonical rows are not in the $(6\epsilon/\gamma,2\epsilon)$ neighborhood of $M'^j$. Thus by Claim~\ref{claim:repsentationwitherror} the row is close to the convex hull of canonical rows and cannot be a robust loner.
\end{proof}

Next we prove the other direction: a canonical row is necessarily a robust loner:

\begin{lemma}
All canonical rows are robust loners.
\end{lemma}

\begin{proof}
Suppose $M'^i$ is a canonical row that corresponds to $W^t$. We first observe that all the rows that are outside the $(6\epsilon/\gamma, 2\epsilon)$ neighborhood of $M'^j$ must have $A_{j,t} < 1-6\epsilon/\gamma$. This is because when $A_{j,t} \ge 1-6\epsilon/\gamma$ we have $M^j - \sum_{k = 1}^r A_{j,t} W^t = \vec{0}$. If we replace $M^j$ by $M'^j$ and $W^t$ by the corresponding canonical row, the distance is still at most $2\epsilon$ and the coefficient on $M'^i$ is at least $1-6\epsilon/\gamma$. By definition the corresponding row $M'^j$ must be in the neighborhood of $M'^i$.

Now we try to represent $M'^i$ with convex combination of rows that has $A_{j,t} < 1-6\epsilon/\gamma$. However this is impossible because every point in the convex combination will also have weight smaller than $1-6\epsilon/\gamma$ on $W^t$, while $M^i$ has weight 1 on $W^t$. The $\ell_1$ distance between $M_i$ and the convex combination of the $M_j$'s where $A_{j,t} < 1-6\epsilon/\gamma$ is a least $6\epsilon$ by robust simplicial property. Even when the points are perturbed by $\epsilon$ (in $\ell_1$) the distance can change by at most $2\epsilon$ and is still more than $2\epsilon$. Therefore $M'^i$ is a robust loner.
\end{proof}

Now we can prove the main theorem of this section:

\vspace{0.75pc}

\begin{proof}
Suppose we know $\gamma$ and $100\epsilon < \gamma$.When $\gamma$ is so small we have the following claim:

\begin{claim}
If $A_{j,t}$ and $A_{i,l}$ is at least $1-10\epsilon/\gamma$, and $t\ne l$, then $M'^j$ cannot be $(10\epsilon/\gamma, 2\epsilon)$-close to $M'^i$ and vice versa.
\end{claim}

The proof is almost identical to Lemma~\ref{lem:farfromcanonical}. Also, the canonical row that corresponds to $W^t$ is $(10\epsilon/\gamma, 2\epsilon)$ close to all rows with $A_{j,t} \ge 1 - 10\epsilon/\gamma$. Thus if we connect two robust loners when one is $(10\epsilon/\gamma, 2\epsilon)$ close to the other, the connected component of the graph will exactly be a partition according to the row in $W$ that the robust loner is close to. We pick one robust loner in each connected component to get the almost anchor words.

Now suppose we don't know $\gamma$. In this case the problem is we don't know what is the right size of neighborhood to look at. However, since we know $\gamma > 100\epsilon$, we shall first run the algorithm with $\gamma = 100\epsilon$ to get 
$r$ rows $W'$ that are very close to the true rows in $W$. It is not hard to show that these rows are at least $\gamma/2$ robustly simplicial and at most $\gamma+2\epsilon$ robustly simplicial. Therefore we can compute the $\gamma(W')$ parameter for this set of rows and use $\gamma(W')-2\epsilon$ as the $\gamma$ parameter.
\end{proof}

\vspace*{-0.1in}

\section{Maximum Likelihood Estimation is Hard}\label{sec:mle}

Here we prove that computing the Maximum Likelihood Estimate (MLE) of the parameters of a topic model is $NP$-hard. We call this problem the Topic Model Maximum Likelihoood Estimation (TM-MLE) problem:

\begin{definition}[TM-MLE] Given $m$ documents and a target of $r$ topics, the TM-MLE problem asks to compute the topic matrix $A$ that has the largest probability of generating the observed documents (when the columns of $W$ are generated by a uniform Dirichlet distribution). 
\end{definition}

Surprisingly, this appears to be the first proof that computing the MLE estimate in a topic model is indeed computationally hard, although its hardness is certainly to be expected. On a related note, Sontag and Roy \cite{SR} recently proved that {\em given the topic matrix} and a document, computing the Maximum A Posteriori (MAP) estimate for the distribution on topics that generated this document is $NP$-hard. Here we will establish that TM-MLE is $NP$-hard via a reduction from the MIN-BISECTION problem: In MIN-BISECTION the input is a graph with $n$ vertices ($n$ is an even integer), and the goal is to partition the vertices into two equal sized sets of $n/2$ vertices each so as to minimize the number of edges crossing the cut.

\begin{theorem}
There is a polynomial time reduction from MIN-BISECTION to TM-MLE ($r = 2$). 
\end{theorem}

\begin{proof}
Suppose we are given an instance $G$ of the MIN-BISECTION problem with $n$ vertices and $m$ edges. We will now define an instance of the TM-MLE problem. First, we set the number of words to be $n$. For each word $i$, we construct $N = \lceil 200m^3\log n \rceil$ documents each of which contain the word $i$ twice and no other words. For each edge in the graph $G$, we construct a document whose two words correspond to the endpoints of the edge.

Suppose that $x = (x_1, x_2)^T$ is generated by the Dirichlet distribution $Dir(1,1)$. Consequently the probability that words $i$ and $j$ appear in a document with only two words is exactly $(A^ix)\cdot (A^j x)$. We can take the expectation of this term over the Dirichlet distribution $Dir(1,1)$ and hence the probability that a document (with exactly two words) contains the words $i$ and $j$ is $$ \E[(A^ix)\cdot (A^j x)] = \frac{1}{3} (A^i \cdot A^j) + \frac{1}{6} (A^i_1A^j_2 + A^j_1 A^i_2)$$ In the TM-MLE problem, our goal is to maximize the following objective function (which is the $\log$ of the probability of generating the collection of documents):

$$
OBJ = \sum_{\mbox{document $= \{i, j\}$}} \log \left[ \frac{1}{3} (A^i \cdot A^j) + \frac{1}{6} (A^i_1A^j_2 + A^j_1 A^i_2) \right].
$$

For any bisection, we define a canonical solution: the first topic is uniform on all words on one side of the bisection and the second topic is uniform on all words on the other side of the bisection.To prove the correctness of our reduction, a key step is to show that any candidate solution to the MLE problem must be close to a canonical solution. In particular, we show the following:

\begin{enumerate}
\item The rows $A^i$ have almost the same $\ell_1$ norm.
\item In each row $A^i$, almost all of the weight will be in one of the two topics.
\end{enumerate}

Indeed, canonical solutions have large objective value. Any canonical solution has objective value at least $ - Nn \log 3n^2/4 - m \log 3n^2/2$ (this is because documents with same words contribute $-\log 3n^2/4$ and documents with different words contribute at least $-\log 3n^2/2$).

Recall, in our reduction $N$ is large. Roughly, if one of the rows has $\ell_1$ norm that is bounded away from $2/n$ by at least 
 $1/20nm$, the contribution (to the objective function) of documents with a repeated word will decrease significantly and the solution cannot be optimal. To prove this we use the fact that the function $\log x^2 = 2 \log x$ is concave. Therefore when one of the rows has $\ell_1$ norm more than $2/n+1/20nm$, the optimal value for documents with a repeated word will be attained when all other rows have the same $\ell_1$ norm $2/n-1/20nm(n-1)$. Using a Taylor expansion, we conclude that the sum of terms for documents with a repeated word will decrease by at least $N/50m^2$ which is much larger than any effect the remaining $m$ documents can recoup.
In fact, an identical argument establishes that in each row $A^i$, the topic with smaller weight will always have weight smaller than $1/20nm$. 

Now we claim that among canonical solutions, the one with largest objective value corresponds to a minimum bisection. The proof follows from the observation that the value of the objective function is $- Nn \log 3n^2/4 - k \log 3n^2/2 - (m-k) \log 3n^2/4$ for canonical solutions, where $k$ is the number of edges cut by the bisection. In particular, the objective function of the minimum bisection will be at least an additive $\log 2$ larger than the objective function of a non-minimum bisection. 

However, even if the canonical solution is perturbed by $1/20nm$, the objective function will only change by at most $m\cdot 1/10m = 1/10$, which is much smaller than $\frac{\log 2}{2}$. And this completes our reduction.
\end{proof}

We remark that the canonical solutions in our reduction are all {\em separable}, and hence this reduction applies even when the topic matrix $A$ is known (and required) to be separable. So, even in the case of a separable topic matrix, it is $NP$-hard to compute the MLE. 

\vspace*{-0.1in}

\section{Conclusions}

We expect that versions of our algorithm may indeed be practical, and
are investigating this possibility. Our machine learning colleagues
suggest that real-life topic matrices satisfy even stronger {\em
  separability} assumptions, e.g., the presence of {\em many} anchor
words per topic instead of a single one. This is a promising suggestion,
but leveraging it in our algorithm is an open problem.

Is separability necessary for allowing polynomial-time algorithms for
the learning problems considered here? In other words, is the problem
difficult if the topic matrix $A$ is not {\em separable}? Average-case
intractability  seems more plausible here
than NP-completeness. 

\newpage

\section*{Acknowledgements}
We thank Dave Blei, Ravi Kannan, David Minmo, Sham Kakade, David Sontag for many helpful discussions throughout various stages of this work.

\begin{thebibliography}{99}

\bibitem{AHK}
A. Anandkumar, D. Hsu and S. Kakade.
\newblock A method of moments for mixture models and hidden Markov models.
\newblock Arxiv, 2012. 

\bibitem{AGKM}
S. Arora, R. Ge, R. Kannan and A. Moitra.
\newblock Computing a nonnegative matrix factorization \--- provably.
\newblock {\em STOC} 2012, to appear.  

\bibitem{AFKMS}
Y. Azar, A. Fiat, A. Karlin, F. McSherry and J. Saia. 
\newblock Spectral analysis of data.
\newblock {\em STOC}, pp. 619--626, 2001. 

\bibitem{survey}
D. Blei.
\newblock Introduction to probabilistic topic models.
\newblock {\em Communications of the ACM}, pp. 77--84, 2012.

\bibitem{Blei}
D. Blei.
\newblock Personal communication.

\bibitem{LDA}
D. Blei, A. Ng and M. Jordan.
\newblock Latent Dirichlet Allocation.
\newblock {\em Journal of Machine Learning Research}, pp. 993--1022, 2003. Preliminary version in {\em NIPS} 2001. 


\bibitem{BL1}
D. Blei and J. Lafferty.
\newblock A correlated topic model of Science.
\newblock {\em Annals of Applied Statistics}, pp. 17--35, 2007. 

\bibitem{BL2}
D. Blei and J. Lafferty.
\newblock Dynamic topic models.
\newblock {\em ICML}, pp. 113--120, 2006. 

\bibitem{CR93}
J. Cohen and U. Rothblum.
\newblock Nonnegative ranks, decompositions and factorizations of nonnegative matices.
\newblock {\em Linear Algebra and its Applications}, pp. 149--168, 1993.

\bibitem{LSI}
S. Deerwester, S. Dumais, T. Landauer, G. Furnas and R. Harshman. 
\newblock Indexing by latent semantic analysis.
\newblock {\em JASIS}, pp. 391--407, 1990.

\bibitem{DLR}
A.~P. Dempster, N.~M. Laird, and D.~B. Rubin.
\newblock Maximum likelihood from incomplete data via the EM Algorithm.
\newblock {\em J. Roy. Statist. Soc. Ser. B}, pp. 1--38, 1977.

\bibitem{DS}
D. Donoho and V. Stodden.
\newblock When does non-negative matrix factorization give the correct decomposition into parts?
\newblock {\em NIPS}, 2003.

\bibitem{GV}
G. Golub and C. van Loan.
\newblock {\em Matrix Computations}
\newblock The Johns Hopkins University Press, 1996.

\bibitem{Lass}
N. Gravin, J. Lasserre, D. Pasechnik and S. Robins.
\newblock The inverse moment problem for convex polytopes.
\newblock {\em Discrete and Computation Geometry}, 2012, to appear. 

\bibitem{Hof}
T. Hofmann.
\newblock Probabilistic latent semantic analysis.
\newblock {\em UAI}, pp. 289--296, 1999. 

\bibitem{Hoy}
P. Hoyer.
\newblock Non-negative matrix factorization with sparseness constraints.
\newblock {\em Journal of Machine Learning Research}, pp. 1457--1469, 2004.

\bibitem{HKO}
A. Hyv\"arinen, J. Karhunen and E. Oja.
\newblock {\em Independent Component Analysis}.
\newblock Wiley Interscience, 2001.

\bibitem{JGJS}
M. Jordan, Z. Ghahramani, T. Jaakola and L. Saul.
\newblock Introduction to variational methods for graphical models.
\newblock {\em Machine Learning}, pp. 183--233, 1999. 

\bibitem{KS2}
J. Kleinberg and M. Sandler.
\newblock Using mixture models for collaborative filtering. 
\newblock {\em JCSS}, pp. 49--69, 2008. Preliminary version in {\em STOC} 2004. 

\bibitem{KS1}
J. Kleinberg and M. Sandler.
\newblock Convergent algorithms for collaborative filtering. 
\newblock {\em ACM EC}, pp. 1--10, 2003. 

\bibitem{KRRT}
R. Kumar, P. Raghavan, S. Rajagopalan and A. Tomkins.
\newblock Recommendation systems: a probabilistic analysis.
\newblock {\em JCSS}, pp. 42--61, 2001. Preliminary version in {\em FOCS} 1998.

\bibitem{LS99}
D. Lee and H. Seung.
\newblock Learning the parts of objects by non-negative matrix factorization.
\newblock {\em Nature}, pp. 788-791, 1999. 

\bibitem{LS00}
D. Lee and H. Seung.
\newblock Algorithms for non-negative matrix factorization.
\newblock {\em NIPS}, pp. 556--562, 2000.

\bibitem{LM}
W. Li and A. McCallum.
\newblock Pachinko Allocation: DAG-structured mixture models of topic correlations.
\newblock {\em ICML}, pp. 633-640, 2007.

\bibitem{Mat}
J. Matousek.
\newblock {\em Lectures on Discrete Geometry}.
\newblock Springer, 2002. 

\bibitem{McS}
F. McSherry.
\newblock Spectral partitioning of random graphs.
\newblock {\em FOCS}, pp. 529--537, 2001. 

\bibitem{PRTV}
C. Papadimitriou, P. Raghavan, H. Tamaki and S. Vempala.
\newblock Latent semantic indexing: a probabilistic analysis.
\newblock {\em JCSS}, pp. 217--235, 2000. Preliminary version in {\em PODS} 1998.

\bibitem{RW}
R.~A. Redner and H.~F. Walker.
\newblock Mixture densities, maximum likelihood and the EM Algorithm.
\newblock {\em SIAM Rev. }, pp. 195-239, 1984.

\bibitem{SR}
D. Sontag and D. Roy.
\newblock Complexity of inference in Latent Dirichlet Allocation.
\newblock {\em NIPS}, pp. 1008--1016, 2011. 

\bibitem{XLG}
W. Xu and X. Liu and Y. Gong.
\newblock Document clustering based on non-negative matrix factorization.
\newblock {\em SIGIR}, pp. 267--273, 2003. 

\bibitem{Vav}
S. Vavasis.
\newblock On the complexity of nonnegative matrix factorization.
\newblock {\em SIAM Journal on Optimization}, pp. 1364-1377, 2009.

\bibitem{WJ}
M. Wainwright and M. Jordan.
\newblock Graphical models, exponential families, and variational inference.
\newblock {\em Foundations and Trends in Machine Learning}, pp. 1--305, 2008. 

\end{thebibliography}
\end{document}